\def\year{2017}\relax
\documentclass[letterpaper]{article}
\usepackage{aaai17}
\usepackage{times}
\usepackage{helvet}
\usepackage{courier}
\usepackage{url}
\usepackage{amsmath,amssymb,amsthm,bm}
\usepackage{mathrsfs}
\usepackage{times}
\usepackage{graphicx} 
\usepackage{subfigure}
\usepackage{algorithm}
\usepackage{algorithmic}
\def\a{{\bf a}}
\def\b{{\bf b}}

\def\p{{\bf p}}

\def\u{{\bf u}}
\def\v{{\bf v}}
\def\w{{\bf w}}
\def\x{{\bf x}}
\def\y{{\bf y}}

\def\B{{\bf B}}

\def\P{{\bf P}}

\def\0{{\bf 0}}
\def\1{{\bf 1}}
\def\2{{\bf 2}}
\def\3{{\bf 3}}
\def\4{{\bf 4}}
\def\5{{\bf 5}}
\def\6{{\bf 6}}
\def\7{{\bf 7}}
\def\8{{\bf 8}}
\def\9{{\bf 9}}

\def\EB{{\mathbb E}}

\def\RB{{\mathbb R}}

\newtheorem{theorem}{Theorem}
\newtheorem{lemma}{Lemma}

\newtheorem{assumption}{Assumption}

\begin{document}
%
\title{Lock-Free Optimization for Non-Convex Problems}
\author{Shen-Yi Zhao, Gong-Duo Zhang \and Wu-Jun Li\\
National Key Laboratory for Novel Software Technology \\
Department of Computer Science and Technology, Nanjing University, China \\
\texttt{\{zhaosy, zhanggd\}@lamda.nju.edu.cn, liwujun@nju.edu.cn}
}
\maketitle
\begin{abstract}
Stochastic gradient descent~(SGD) and its variants have attracted much attention in machine learning due to their efficiency and effectiveness for optimization. To handle large-scale problems, researchers have recently proposed several lock-free strategy based parallel SGD~(LF-PSGD) methods for multi-core systems. However, existing works have only proved the convergence of these LF-PSGD methods for convex problems. To the best of our knowledge, no work has proved the convergence of the LF-PSGD methods for non-convex problems. In this paper, we provide the theoretical proof about the convergence of two representative LF-PSGD methods, Hogwild! and AsySVRG, for non-convex problems. Empirical results also show that both Hogwild! and AsySVRG are convergent on non-convex problems, which successfully verifies our theoretical results.
\end{abstract}

\section{Introduction}
Many machine learning models can be formulated as the following optimization problem:
\begin{align}\label{problem}
	\underset{\w}{\min}~\frac{1}{n}\sum_{i=1}^n f_i(\w),
\end{align}
where $\w$ is the parameter to learn~(optimize), $n$ is the number of training instances, $f_i(\w)$ is the loss defined on instance $i$. For example, assuming we are given a set of labeled instances
$\{(\x_i, y_i)|i=1,2,\ldots, n\}$, where $\x_i \in \RB^d$ is the
feature vector and $y_i \in \{1,-1\}$ is the label of $\x_i$, $f_i(\w)$ can be $\log(1+e^{-y_i\x_i^T\w}) + \frac{\lambda}{2}\|\w\|^2$ which is known as the regularized loss in logistic regression~(LR). We can also take $f_i(\w)$ to be $\max\{0, 1-y_i\x_i^T\w\} + \frac{\lambda}{2}\|\w\|^2$ which is known as the regularized loss in support vector machine~(SVM). Here, $\lambda$ is the regularization hyper-parameter. Moreover, many other machine learning models, including neural networks~\cite{DBLP:conf/nips/KrizhevskySH12}, matrix factorization~\cite{DBLP:journals/computer/KorenBV09}, and principal component analysis~(PCA)~\cite{pca} and so on, can also be formulated as that in~(\ref{problem}).

When the problem in~(\ref{problem}) is large-scale, i.e., $n$ is large, researchers have recently proposed stochastic gradient descent~(SGD) and its variants like SVRG~\cite{DBLP:conf/nips/Johnson013} to solve it. Many works~\cite{DBLP:conf/nips/RouxSB12,DBLP:journals/jmlr/Shalev-Shwartz013,DBLP:conf/nips/Johnson013} have found that SGD-based methods can achieve promising performance in large-scale learning problems. According to the implementation platforms or systems, existing SGD-based methods can be divided into three categories: sequential SGD~(SSGD) methods, parallel SGD~(\mbox{PSGD}) methods, and distributed SGD~(DSGD) methods. SSGD methods are designed for a single thread on a single machine, PSGD methods are designed for multi-core~(multi-thread) on a single machine with a shared memory\footnote{In some literatures, PSGD refers to the methods implemented on both multi-core and multi-machine systems. In this paper, PSGD only refers to the methods implemented on multi-core systems with a shared memory.}, and DSGD methods are designed for multiple machines.



When the problem in~(\ref{problem}) is convex, the SGD methods, including SSGD~\cite{DBLP:conf/nips/RouxSB12,DBLP:journals/jmlr/Shalev-Shwartz013,DBLP:conf/nips/Johnson013}, PSGD~\cite{recht2011hogwild} and DSGD~\cite{DBLP:conf/nips/JaggiSM14,DBLP:conf/osdi/LiAPSAJLSS14,DBLP:conf/kdd/XingHDKWLZXKY15,Zhang2015distributed}, have achieved very promising empirical performance. Furthermore, good theoretical results about the convergence of the SGD methods are also provided by these existing works.


In many real applications, the problems to optimize can be non-convex. For example, the problems for the neural networks are typically non-convex. Because many researchers~~\cite{DBLP:conf/osdi/LiAPSAJLSS14,DBLP:conf/kdd/XingHDKWLZXKY15} find that the SGD methods can also achieve good empirical results for non-convex problems, theoretical proof about the convergence of SGD methods for non-convex problems has recently attracted much attention. Some progress has been achieved. For example, the works in~\cite{DBLP:journals/siamjo/GhadimiL13a,sashanksvrg16,xingguolisvrg16,zeyuansvrg16,zeyuanimprovedsvrg16} have proved the convergence of the sequential SGD and its variants for non-convex problems. There are also some other theoretical results for some particular non-convex problems, like PCA~\cite{pca,ohadpca16sgd,ohadpca16svrg} and matrix factorization~\cite{DBLP:conf/icml/SaRO15}. But all these works are only for SSGD methods.

There have appeared only two works~\cite{DBLP:conf/nips/LianHLL15,DBLP:journals/corr/HuoH16} which propose PSGD methods for non-convex problems with theoretical proof of convergence. However, the PSGD methods in~\cite{DBLP:conf/nips/LianHLL15} need write-lock or atomic operation for the memory to prove the convergence~\footnote{Although the implementation of AsySG-incon in~\cite{DBLP:conf/nips/LianHLL15} is lock-free, the theoretical analysis about the convergence of AsySG-incon is based on an assumption that no over-writing happens, i.e., the theoretical analysis is not for the lock-free case.}. Similarly, the work in~\cite{DBLP:journals/corr/HuoH16} also does not prove the convergence for the lock-free case in our paper. Recent works~\cite{recht2011hogwild,DBLP:conf/nips/ChaturapDC15,journals/corr/ReddiHSPS15,DBLP:conf/aaai/ZhaoL16} find that lock-free strategy based parallel SGD~(LF-PSGD) methods can empirically outperform lock-based PSGD methods for multi-core systems. Although some existing works~\cite{DBLP:conf/nips/ChaturapDC15,DBLP:conf/aaai/ZhaoL16} have proved the convergence of these LF-PSGD methods for convex problems, no work has proved the convergence of the LF-PSGD methods for non-convex problems.

In this paper, we provide the theoretical proof about the convergence of two representative LF-PSGD methods, Hogwild!~\cite{recht2011hogwild,DBLP:conf/nips/ChaturapDC15} and AsySVRG~\cite{DBLP:conf/aaai/ZhaoL16}, for non-convex problems. The contribution of this work can be outlined as follows:
\begin{itemize}
  \item Theoretical results show that both Hogwild! and AsySVRG can converge with lock-free strategy for non-convex problems.
  \item Hogwild! gets a convergence rate of $O(1/\sqrt{\tilde{T}})$ for non-convex problems, where $\tilde{T} = p\times T$ is the total iteration number of $p$ threads.
  \item AsySVRG gets a convergence rate of $O(1/\tilde{T})$ for non-convex problems.
  \item To get an $\epsilon$-local optimal solution for AsySVRG, the computation complexity by all threads is $O(n^{\frac{2}{3}}/\epsilon)$, or equivalently the computation complexity of each thread is $O(\frac{n^{\frac{2}{3}}}{p\epsilon})$. This is faster than traditional parallel gradient decent methods whose computation complexity is $O(\frac{n}{p\epsilon})$ for each thread.
  \item Empirical results also show that both Hogwild! and AsySVRG are convergent on non-convex problems, which successfully verifies our theoretical results.
\end{itemize}

\section{Preliminary}
We use $f(\w)$ to denote the objective function in (\ref{problem}), which means $f(\w) = \frac{1}{n}\sum_{i=1}^n f_i(\w)$. And we use $\| \cdot \|$ to denote the $L_2$-norm $\| \cdot \|_2$.

\begin{assumption}\label{ass:smooth}
	The function $f_i(\cdot)$ in (\ref{problem}) is smooth, which means that there exists a constant $L>0$, $\forall \a, \b$,
	\begin{align}
		f_i(\b) \leq f_i(\a) + \nabla f_i(\a)^T(\b-\a) + \frac{L}{2}\| \b - \a \|^2, \nonumber
	\end{align}
	or equivalently
	\begin{align}
		\| \nabla f_i(\b) - \nabla f_i(\a)\| \leq L \| \b -\a \|. \nonumber
	\end{align}
\end{assumption}
This is a common assumption for the convergence analysis of most existing gradient-based methods.

Since we focus on non-convex problems in this paper, it is difficult to get the global solution of~(\ref{problem}) based on the gradient methods. Hence, we use $\| \nabla f(\w) \|^2$ to measure the convergence instead of $f(\w) - \underset{\w}{\min}~f(\w)$.

Here, we give a Lemma which is useful in the convergence analysis of Hogwild! and AsySVRG.

\begin{lemma}\label{techlemma:tech}
Assume $\B$ is a positive semi-definite matrix with the largest eigenvalue less than or equal to 1 and the minimum eigenvalue $\alpha >0$, we have: $\forall \x,\y$,
\begin{align}
-\nabla f(\x)^T \B \nabla f(\y) \leq \frac{L^2}{2}\left\| \x - \y \right\|^2 - \frac{\alpha}{2}\left\| \nabla f(\x) \right\|^2 . \nonumber
\end{align}
\end{lemma}
\begin{proof}
\begin{align}
     &\frac{\alpha}{2}\left\|\nabla f(\x) \right\|^2 - \nabla f(\x)^T \B \nabla f(\y) \nonumber \\
\leq&\frac{1}{2}\left\| \B^{\frac{1}{2}}\nabla f(\x) \right\|^2 - \nabla f(\x)^T \B \nabla f(\y) \nonumber \\
\leq &\frac{1}{2}\left\| \B^{\frac{1}{2}}\nabla f(\x) \right\|^2 - \nabla f(\x)^T \B \nabla f(\y) + \frac{1}{2}\left\| \B^{\frac{1}{2}}\nabla f(\y) \right\|^2 \nonumber \\
=    &\frac{1}{2}\left\| \B^{\frac{1}{2}}(\nabla f(\x) - \nabla f(\y))\right\|^2 \nonumber \\
\leq &\frac{L^2}{2}\left\| \x - \y \right\|^2 .\nonumber
\end{align}
\end{proof}

\section{Hogwild! for Non-Convex Problems}\label{sec:Hogwild}
The Hogwild! method~\cite{recht2011hogwild} is listed in Algorithm~\ref{alg:Hogwild!}. Each thread reads $\w$ from the shared memory, computes a stochastic gradient and updates the $\w$ in the shared memory. Please note that Hogwild! in~\cite{recht2011hogwild} has several variants with locks or lock-free. Here, we only focus on the lock-free variant of Hogwild!, which means that we do not use any locks, either read-lock or write-lock, for all threads.

\begin{algorithm}[!h]
\caption{Hogwild!}
\label{alg:Hogwild!}
\small
\begin{algorithmic}
\STATE Initialization: $p$ threads, initialize $\w_0, \eta$;
\STATE For each thread, do:
\FOR{$l=0,1,2,...,T-1$}
\STATE Read current $\w$ in the shared memory, denoted as $\hat{\w}$;
\STATE Randomly pick up an $i$ from $\left\{ 1,\ldots,n \right\}$ and compute the gradient $\nabla f_i(\hat{\w})$;
\STATE $\w \leftarrow \w - \eta \nabla f_i(\hat{\w})$;
\ENDFOR
\end{algorithmic}
\end{algorithm}

As in~\cite{DBLP:conf/aaai/ZhaoL16}, we can construct an equivalent write sequence $\{\w_t\}$:
\begin{align}\label{seq:write}
	\w_{t+1} = \w_t - \eta \B_t\nabla f_{i_t}(\hat{\w}_t),
\end{align}
where $0 \leq t \leq p \times T$, $\B_t$ is a random diagonal matrix whose diagonal entries are $0$ or $1$. The $\B_t$ is used to denote whether over-writing happens. If the $k$th diagonal entry of $\B_t$ is $0$, it means that the $k$th element in the gradient vector $\nabla f_{i_t}(\hat{\w}_t)$ is overwritten by other threads. Otherwise, that element is not overwritten.

$\hat{\w}_t$ is read by the thread who computes $\nabla f_{i_t}(\hat{\w}_t)$ and has the following format:
\begin{align}\label{read}
	\hat{\w}_t = \w_{a(t)} - \eta\sum_{j=a(t)}^{t-1}\P_{t, j-a(t)}\nabla f_{i_j}(\hat{\w}_j),
\end{align}
where $a(t)$ means that some old stochastic gradients have been completely written on the $\w$ in the shared memory. $\P_{t, j-a(t)}$ is a diagonal matrix whose diagonal entries are $0$ or $1$, which means $\hat{\w}_t$ might include parts of new stochastic gradients.

In the lock-free strategy, we need the following assumptions to guarantee convergence:

\begin{assumption}\label{ass:delay}
$a(t)$ is bounded by: $0 \leq t - a(t) \leq \tau $
\end{assumption}

It means that the old stochastic gradients $\nabla f_{i_0}, \ldots, \nabla f_{i_{t-\tau-1}}$ have been completely written on $\w$ in the shared memory.

\begin{assumption}\label{ass:expB}
We consider the matrix $\B_t$ as a random matrix and $\EB[\B_t|\w_t,\hat{\w}_t] = \B \succ 0$ with the minimum eigenvalue $\alpha > 0$.
\end{assumption}
According to the definition of $\B_t$, it is easy to find $\B_t, \B$ are positive semi-definite matrices and the largest eigenvalue of $\B$ is less than or equal to 1. Assumption \ref{ass:expB} means that the probability that over-writing happens is at most $1-\alpha < 1$ for each write step.

\begin{assumption}\label{ass:independence}
$\B_t$ and $i_t$ are independent.	
\end{assumption}
Since $i_t$ is the random index selected by each thread while $\B_t$ is highly affected by the hardware, the independence assumption is reasonable.

For Hogwild!, the following assumption is also necessary:
\begin{assumption}\label{ass:boundedvariance}
There exists a constant $V$, $\| \nabla f_i(\w) \| \leq V, i=1,\ldots,n$.
\end{assumption}
For convenience, in this section, we denote $$q(\x) = \frac{1}{n}\sum_{i=1}^n \| f_i(\x) \|^2.$$ It is easy to find that $\EB q(\hat{\w}_t) = \EB[\| \nabla f_{i_t}(\hat{\w}_t)\|^2]$ and note that when $\x$ is close to some stationary point, $q(\x)$ may still be far away from $0$. Hence, it is not a variance reduction method and we need to control the variance of the stochastic gradient.

The difficulty of the analysis is $\w_t\neq \hat{\w}_t$. Here, we give the following Lemmas~\footnote{The proof of some Lemmas can be found in the supplementary material, which can be downloaded from \url{http:
//cs.nju.edu.cn/lwj/paper/LFnonConvex_sup.pdf}.}:
\begin{lemma}\label{techlemma:Hog}
In Hogwild!, we have $\EB q(\hat{\w}_{t}) \leq \rho \EB q(\hat{\w}_{t+1})$  if $\rho, \eta$ satisfy
\begin{align}
\frac{1}{1-\eta-\frac{9\eta(\tau+1)L^2(\rho^{\tau+1}-1)}{\rho-1}} \leq \rho. \nonumber
\end{align}
\end{lemma}

\begin{lemma}\label{Hoglem:gap of wr}
With the condition about $\rho, \eta$ in Lemma \ref{techlemma:Hog}, we have
\begin{align}
  \EB\|\w_t - \hat{\w}_t \|^2 \leq \frac{4\eta^2 \tau\rho(\rho^{\tau}-1)}{\rho-1}\EB q(\hat{\w}_t)
\end{align}
\end{lemma}

Combining with Assumption \ref{ass:boundedvariance}, we can find that the gap of the write sequence and read sequence can always be bounded by a constant $\frac{4\eta^2V^2 \tau\rho(\rho^{\tau}-1)}{\rho-1}$.


\begin{theorem}
	Let $A = \frac{2f(\w_0)}{\alpha}$ and $B=2V^2(\frac{2\tau L^2\eta\rho(\rho^{\tau}-1)}{\alpha(\rho-1)}+\frac{L}{2\alpha})$. If we take the stepsize $\eta = \sqrt{\frac{A}{\tilde{T}B}}$, where $\tilde{T} = p\times T$, we can get the following result:
	\begin{align}
    	\frac{1}{\tilde{T}}\sum_{t=0}^{\tilde{T}-1}\EB \| \nabla f(\w_t) \|^2 \leq \sqrt{\frac{AB}{\tilde{T}}}. \nonumber
    \end{align}
\end{theorem}
\begin{proof}
	According to Assumption \ref{ass:smooth}, we have
	\begin{align}
		     &\EB[f(\w_{t+1})|\w_t,\hat{\w}_t] \nonumber \\
	    \leq &f(\w_t) - \eta\EB[\nabla f(\w_t)^T\B_t\nabla f_{i_t}(\hat{\w}_t)|\w_t,\hat{\w}_t] \nonumber \\
             &+ \frac{L\eta^2}{2}\EB[\| \nabla f_{i_t}(\hat{\w}_t) \|^2|\w_t,\hat{\w}_t] \nonumber \\
	    =    &f(\w_t) - \eta\nabla f(\w_t)^T\B\nabla f(\hat{\w}_t) \nonumber \\
             &+ \frac{L\eta^2}{2}\EB[\| \nabla f_{i_t}(\hat{\w}_t) \|^2|\w_t,\hat{\w}_t] \nonumber \\
	    \leq &f(\w_t) - \frac{\alpha\eta}{2}\| \nabla f(\w_t) \|^2 + \frac{L^2\eta}{2}\| \w_t - \hat{\w}_t\|^2 \nonumber \\
             &+\frac{L\eta^2}{2}\EB[\| \nabla f_{i_t}(\hat{\w}_t) \|^2|\w_t,\hat{\w}_t], \nonumber
	\end{align}
	where the first equality uses Assumption \ref{ass:independence}, the second inequality uses Lemma~\ref{techlemma:tech}.
    Taking expectation on the above inequality, we obtain
    \begin{align}
    	 &\EB f(\w_{t+1}) \nonumber \\
    \leq &\EB f(\w_t) - \frac{\alpha\eta}{2}\EB\| \nabla f(\w_t) \|^2 + \frac{L^2\eta}{2}\EB\| \w_t - \hat{\w}_t\|^2 \nonumber \\
             &+\frac{L\eta^2V^2}{2} \nonumber \\
    \leq &\EB f(\w_t) - \frac{\alpha\eta}{2}\EB\| \nabla f(\w_t) \|^2 \nonumber \\
         &+ \eta^2V^2(\frac{2\tau L^2\eta\rho(\rho^{\tau}-1)}{\rho-1}+\frac{L}{2}), \nonumber
    \end{align}
    where the first inequality uses Assumption \ref{ass:boundedvariance} and second inequality uses Lemma \ref{Hoglem:gap of wr}.
    Summing the above inequality from $t=0$ to $\tilde{T}-1$, we get
    \begin{align}
    	     &\sum_{t=0}^{\tilde{T}-1}\EB \| \nabla f(\w_t) \|^2 \nonumber \\
        \leq &\frac{2}{\alpha\eta}f(\w_0) + 2\eta\tilde{T}V^2(\frac{2\tau L^2\eta\rho(\rho^{\tau}-1)}{\alpha(\rho-1)}+\frac{L}{2\alpha}). \nonumber
    \end{align}
    For convenience, let $A = \frac{2f(\w_0)}{\alpha}$ and $B=2V^2(\frac{2\tau L^2\eta\rho(\rho^{\tau}-1)}{\alpha(\rho-1)}+\frac{L}{2\alpha})$, which are two bounded constants. If we take the stepsize $\eta = \sqrt{\frac{A}{\tilde{T}B}}$, we get
    \begin{align}
    	\frac{1}{\tilde{T}}\sum_{t=0}^{\tilde{T}-1}\EB \| \nabla f(\w_t) \|^2 \leq \sqrt{\frac{AB}{\tilde{T}}}. \nonumber
    \end{align}
\end{proof}

Hence, our theoretical result shows that Hogwild! with lock-free strategy gets a convergence rate of $O(1/\sqrt{\tilde{T}})$ for non-convex problems, where $\tilde{T}=p\times T$ is the total iteration number of $p$ threads.

\section{AsySVRG for Non-Convex Problems}\label{sec:AsySVRG}
The AsySVRG method~\cite{DBLP:conf/aaai/ZhaoL16} is listed in Algorithm~\ref{alg:AsySVRG}. AsySVRG provides a lock-free parallel strategy for the original sequential SVRG~\cite{DBLP:conf/nips/Johnson013}. Compared with Hogwild!, AsySVRG includes the full gradient to get a variance reduced stochastic gradient, which has been proved to have linear convergence rate on strongly convex problems~\cite{DBLP:conf/aaai/ZhaoL16}. In this section, we will prove that AsySVRG is also convergent for non-convex problems, and has faster convergence rate than Hogwild! on non-convex problems.

\begin{algorithm}[!thb]
\caption{AsySVRG}
\small
\label{alg:AsySVRG}
\begin{algorithmic}
\STATE Initialization: $p$ threads, initialize $\w_0, \eta$;
\FOR{$t=0,1,2,...T-1$}
\STATE $\u_0 = \w_t$;
\STATE All threads parallelly compute the full gradient $\nabla f(\u_0) = \frac{1}{n} \sum_{i=1}^n \nabla f_i(\u_0)$;
\STATE $\u = \w_t$;
\STATE For each thread, do:
\FOR{$j=0$ to $M-1$}
\STATE Read current value of $\u$, denoted as  $\hat{\u}$, from the shared memory. And randomly pick up an $i$ from $\left\{ 1,\ldots,n \right\}$;
\STATE Compute the update vector: $\hat{\v} = \nabla f_{i}(\hat{\u}) - \nabla f_{i}(\u_0) + \nabla f(\u_0)$;
\STATE $\u \leftarrow \u - \eta \hat{\v}$;
\ENDFOR
\STATE Take $\w_{t+1}$ to be the current value of $\u$ in the shared memory;
\ENDFOR
\end{algorithmic}
\end{algorithm}

Similar to the analysis in the last section, we construct an equivalent write sequence $\{\u_{t,m}\}$ for the $t^{th}$ outer-loop:
\begin{align}
	\u_{t,0} &= \w_t, \nonumber \\
	\u_{t,m+1}& = \u_{t,m} - \eta \B_{t,m}\hat{\v}_{t,m},\label{AsySVRG:update}
\end{align}
where $\hat{\v}_{t,m} = \nabla f_{i_{t,m}}(\hat{\u}_{t,m}) - \nabla f_{i_{t,m}}(\u_{t,0}) + \nabla f(\u_{t,0})$. $\B_{t,m}$ is a diagonal matrix whose diagonal entries are $0$ or $1$. And $\hat{\u}_{t,m}$ is read by the thread who computes $\hat{\v}_{t,m}$. It has the following format:
\begin{align}
	\hat{\u}_{t,m} = \u_{t,a(m)} - \eta \sum_{j=a(m)}^{m-1} \P_{m, j-a(m)}^{(t)} \hat{\v}_{t,j}, \nonumber
\end{align}
where $\P_{m, j-a(m)}^{(t)}$ is a diagonal matrix whose diagonal entries are $0$ or $1$. Note that according to (\ref{AsySVRG:update}), $\u_{t,\tilde{M}} = \w_{t+1}$ since all the stochastic gradients have been written on $\w$ at the end of the $t^{th}$ outer-loop. Here, we also need the assumptions: $0\leq m - a(m) \leq \tau$; $\EB[\B_{t,m}|\u_{t,m}, \hat{\u}_{t,m}] = \B \succ 0$ with the minimum eigenvalue $\alpha > 0$; $\B_{t,m}$ and $i_{t,m}$ are independent. These assumptions are similar to those in the previous section.

For convenience, let $\p_i(\x) = \nabla f_i(\x) - \nabla f_i(\u_{t,0}) + \nabla f(\u_{t,0})$, and in this section, we denote $$q(\x) = \frac{1}{n}\sum_{i=1}^n\| \p_i(\x) \|^2.$$ It easy to find that $\EB q(\hat{\u}_{t,m}) = \EB[\|\hat{\v}_{t,m}\|^2]$.

The difference between Hogwild! and AsySVRG is the stochastic gradient and we have the following Lemmas which lead to fast convergence rate of AsySVRG:
\begin{lemma}\label{svrglem:var reduce}
    $\forall \x$, we have
	\begin{align}
	    q(\x) \leq 2L^2\| \x - \u_{t,0} \|^2 + 2\|\nabla f(\x)\|^2. \nonumber
    \end{align}
\end{lemma}
\begin{proof}
\begin{align}
q(\x)& =    \frac{1}{n}\sum_{i=1}^n\| \nabla f_i(\x) - \nabla f_i(\u_{t,0}) + \nabla f(\u_{t,0}) \|^2 \nonumber \\
	  \leq &\frac{2}{n}\sum_{i=1}^n\| \nabla f_i(\x) - \nabla f_i(\u_{t,0}) + \nabla f(\u_{t,0}) - \nabla f(\x)\|^2 \nonumber \\
           &+ 2\|\nabla f(\x)\|^2 \nonumber \\
      \leq &\frac{2}{n}\sum_{i=1}^n\| \nabla f_i(\x) - \nabla f_i(\u_{t,0})\|^2 + 2\|\nabla f(\x)\|^2 \nonumber \\
	  \leq &2L^2\| \x - \u_{t,0} \|^2 + 2\|\nabla f(\x)\|^2 .\nonumber
	\end{align}
\end{proof}
According to Lemma \ref{svrglem:var reduce}, we can find that AsySVRG is a variance reduction method for non-convex problems, because when $\hat{\u}_{t,m}, \u_{t,0}$ get close to some stationary point, $q(\hat{\u}_{t,m})$ gets close to $0$. And hence we do not need the bounded gradient assumption for the convergence proof.


Since $\u_{t,m} \neq \hat{\u}_{t,m}$, the difficulty of convergence analysis lies in the gap between $\u_{t,m}$ and $\hat{\u}_{t,m}$, and the relation between $q(\hat{\u}_{t,m})$ and $q(\u_{t,m})$.
\begin{lemma}\label{techlemma:svrg}
In AsySVRG, we have $\EB q(\hat{\u}_{t,m}) < \rho \EB q(\hat{\u}_{t,m+1})$  if we choose $\rho$ and $\eta$ to satisfy that
\begin{align}
\frac{1}{1-\eta-\frac{9\eta(\tau+1)L^2(\rho^{\tau+1}-1)}{\rho-1}}\leq\rho. \nonumber
\end{align}
\end{lemma}

\begin{lemma}\label{svrglem:gap of wr}
With the condition about $\rho, \eta$ in Lemma \ref{techlemma:svrg}, we have
\begin{align}
  \EB \| \u_{t,m} - \hat{\u}_{t,m} \|^2 \leq \frac{4\eta^2\tau \rho(\rho^\tau-1)}{\rho-1}\EB q(\hat{\u}_{t,m}).
\end{align}
\end{lemma}

\begin{lemma}\label{svrglem:expqum}
With the condition about $\rho, \eta$ in Lemma \ref{techlemma:svrg}, we have $\EB q(\hat{\u}_{t,m}) < \rho \EB q(\u_{t,m})$.
\end{lemma}

Combining Lemma \ref{svrglem:gap of wr} and Lemma \ref{svrglem:expqum}, we can directly obtain:
\begin{align}\label{bound of wr}
  \EB\left\| \hat{\u}_{t,m} - \u_{t,m} \right\|^2 \leq \frac{4\eta^2\tau \rho^2(\rho^\tau-1)}{\rho-1}\EB q(\u_{t,m}).
\end{align}

\begin{theorem}\label{the:AsySVRG}
We define $c_m = c_{m+1}(1+\beta\eta) + 2L^2\eta^2 h_{m+1}$, $h_m = (\frac{\eta L^2}{2} + \frac{2c_m\eta}{\beta})\frac{4\tau\rho^2(\rho^{\tau}-1)}{\rho-1} + (c_m\rho + \frac{L\rho}{2})$ with $c_0, \beta>0$. Furthermore, we choose $c_0, \eta, \beta$ such that $\gamma = \min \frac{\alpha\eta}{2} - \frac{2c_{m+1}\eta}{\beta} - 2\eta^2 h_{m+1} > 0$ and $c_{\tilde{M}} = 0$, where $\tilde{M} = M \times p$. Then we have
\begin{align}
	\frac{1}{T\tilde{M}}\sum_{t=0}^{T-1}\sum_{m=0}^{\tilde{M}-1}\EB\| \nabla f(\u_{t,m}) \|^2 \leq \frac{\EB f(\w_0) - \EB f(\w_{T})}{T\tilde{M}\gamma}. \nonumber
\end{align}
\end{theorem}
\begin{proof}
In the $t^{th}$ outer-loop, similar to ~\cite{sashanksvrg16}, we define $R_{t,m}$ as follows
\begin{align}
	R_{t,m} = f(\u_{t,m}) + c_m \| \u_{t,m} - \u_{t,0} \|^2. \nonumber
\end{align}

Then $\forall \beta>0$,
\begin{align}\label{ineq:part1}
	 &\EB[\| \u_{t,m+1} - \u_{t,0} \|^2|\u_{t,m},\hat{\u}_{t,m}] \nonumber \\
\leq &\EB\| \u_{t,m+1} - \u_{t,m} \|^2 + \| \u_{t,m} - \u_{t,0} \|^2 \nonumber \\
     &- 2\eta(\EB\B_{t,m}\hat{\v}_{t,m})^T(\u_{t,m} - \u_{t,0}) \nonumber \\
\leq &\eta^2 \EB\| \hat{\v}_{t,m} \|^2 + (1+\beta\eta)\| \u_{t,m} - \u_{t,0} \|^2 \nonumber \\
&+ \frac{\eta}{\beta}\|\nabla f(\hat{\u}_{t,m})\|^2 \nonumber \\
\leq &\eta^2 \EB\| \hat{\v}_{t,m} \|^2 + (1+\beta\eta)\| \u_{t,m} - \u_{t,0} \|^2 \nonumber \\
     &+ \frac{2\eta}{\beta}(\|\nabla f(\u_{t,m})\|+\|\nabla f(\hat{\u}_{t,m})-\nabla f(\u_{t,m})\|^2) \nonumber \\
\leq &\eta^2 \EB\| \hat{\v}_{t,m} \|^2 + (1+\beta\eta)\| \u_{t,m} - \u_{t,0} \|^2 \nonumber \\
     &+ \frac{2\eta}{\beta}(\|\nabla f(\u_{t,m})\|^2 \nonumber \\
     &+L^2\|\hat{\u}_{t,m} - \u_{t,m}\|^2),
\end{align}
where the second inequality uses the fact $2ab \leq \beta a^2+\frac{1}{\beta}b^2$. Since the objective function is $L$-smooth, we have
\begin{align}\label{ineq:part2}
	      &\EB[f(\u_{t,m+1})|\u_{t,m},\hat{\u}_{t,m}] \nonumber \\
	 \leq &- \eta\EB[\nabla f(\u_{t,m})^T\B_{t,m} \nabla f_{i_{t,m}}(\hat{\u}_{t,m})|\u_{t,m},\hat{\u}_{t,m}] \nonumber \\
          &+ f(\u_{t,m}) + \frac{L\eta^2}{2}\EB[\| \hat{\v}_{t,m} \|^2|\u_{t,m},\hat{\u}_{t,m}] \nonumber \\
        = &f(\u_{t,m}) - \eta\nabla f(\u_{t,m})^T\B\nabla f(\hat{\u}_{t,m}) \nonumber \\
          &+ \frac{L\eta^2}{2}\EB[\| \hat{\v}_{t,m} \|^2|\u_{t,m},\hat{\u}_{t,m}] \nonumber \\
	 \leq &f(\u_{t,m}) - \frac{\alpha\eta}{2}\| \nabla f(\u_{t,m}) \|^2 \nonumber \\
&+ \frac{\eta L^2}{2}\| \u_{t,m} - \hat{\u}_{t,m} \|^2 \nonumber \\
          &+ \frac{L\eta^2}{2}\EB[\| \hat{\v}_{t,m} \|^2|\u_{t,m},\hat{\u}_{t,m}],
\end{align}
where the first equality uses the independence of $\B_{t,m}, i_{t,m}$, the second inequality uses Lemma~\ref{techlemma:tech}.
Combining (\ref{ineq:part1}) and (\ref{ineq:part2}), we have
\begin{align}
	 &\EB R_{t,m+1} \nonumber \\
=    &\EB f(\u_{t,m+1}) + c_{m+1} \| \u_{t,m+1} - \u_{t,0} \|^2 \nonumber \\
\leq &\EB f(\u_{t,m}) - (\frac{\alpha\eta}{2} - \frac{2c_{m+1}\eta}{\beta})\EB \| \nabla f(\u_{t,m}) \|^2 \nonumber \\
     &+ (\frac{\eta L^2}{2} + \frac{2c_{m+1}\eta}{\beta})\EB\| \u_{t,m} - \hat{\u}_{t,m} \|^2 \nonumber \\
     &+ c_{m+1}(1+\beta\eta)\EB\| \u_{t,m} - \u_{t,0} \|^2 \nonumber \\
	 &+ \eta^2(c_{m+1} + \frac{L}{2})\EB\| \hat{\v}_{t,m} \|^2 \nonumber \\
\leq &\EB f(\u_{t,m}) - (\frac{\alpha\eta}{2} - \frac{2c_{m+1}\eta}{\beta})\EB\| \nabla f(\u_{t,m}) \|^2 \nonumber \\
	 &+ (\frac{\eta L^2}{2} + \frac{2c_{m+1}\eta}{\beta})\frac{4\tau\eta^2\rho^2(\rho^{\tau}-1)}{\rho-1}\EB q(\u_{t,m}) \nonumber \\
     &+ c_{m+1}(1+\beta\eta)\EB\| \u_{t,m} - \u_{t,0} \|^2 \nonumber \\
     &+ \eta^2(c_{m+1} + \frac{L}{2})\EB\| \hat{\v}_{t,m} \|^2 ,\nonumber
\end{align}
where the last inequality uses equation (\ref{bound of wr}).

For convenience, we use $h_m = (\frac{\eta L^2}{2} + \frac{2c_m\eta}{\beta})\frac{4\tau\rho^2(\rho^{\tau}-1)}{\rho-1} + \rho(c_m + \frac{L}{2})$. Since $\EB[\|\hat{\v}_{t,m}\|^2]=\EB q(\hat{\u}_{t,m}) \leq \rho\EB q(\u_{t,m})$, we have
\begin{align}
	 &\EB R_{t,m+1} \nonumber \\
\leq &\EB f(\u_{t,m}) - (\frac{\alpha\eta}{2} - \frac{2c_{m+1}\eta}{\beta})\EB\| \nabla f(\u_{t,m}) \|^2 \nonumber \\
	 &+ c_{m+1}(1+\beta\eta)\EB\| \u_{t,m} - \u_{t,0} \|^2 + \eta^2 h_{m+1}\EB q(\u_{t,m}) \nonumber \\
\leq &\EB f(\u_{t,m}) \nonumber \\
 &+ [c_{m+1}(1+\beta\eta) + 2L^2\eta^2 h_{m+1}]\EB\| \u_{t,m} - \u_{t,0} \|^2 \nonumber \\
	 & - (\frac{\alpha\eta}{2} - \frac{2c_{m+1}\eta}{\beta} - 2\eta^2 h_{m+1})\EB \| \nabla f(\u_{t,m}) \|^2 , \nonumber
\end{align}
where the second inequality uses Lemma \ref{svrglem:var reduce}. Then we can obtain:
\begin{align}
	&(\frac{\alpha\eta}{2} - \frac{2c_{m+1}\eta}{\beta} - 2\eta^2 h_{m+1})\EB\| \nabla f(\u_m) \|^2 \nonumber \\
&\leq \EB R_m - \EB R_{m+1}, \nonumber
\end{align}
where $c_m = c_{m+1}(1+\beta\eta) + 2L^2\eta^2 h_{m+1}$.

We set $c_0>0$. It is easy to see that $c_m > c_{m+1}$. We can choose $c_0, \eta, \beta$ to make $c_{\tilde{M}} = 0$. Then we have:
\begin{align}
	&\sum_{m=0}^{\tilde{M}-1}\EB\| \nabla f(\u_{t,m}) \|^2 \nonumber \\
&\leq \frac{\EB R_0 - \EB R_{\tilde{M}}}{\gamma} = \frac{\EB f(\w_t) - \EB f(\w_{t+1})}{\gamma}, \nonumber
\end{align}
which is equivalent to
\begin{align}
	\frac{1}{T\tilde{M}}\sum_{t=0}^{T-1}\sum_{m=0}^{\tilde{M}-1}\EB\| \nabla f(\u_{t,m}) \|^2 \leq \frac{\EB f(\w_0) - \EB f(\w_{T})}{T\tilde{M}\gamma}. \nonumber
\end{align}
\end{proof}

\subsection{Computation Complexity}
In Theorem \ref{the:AsySVRG}, we construct a sequence $\{c_m\}$ and need $\gamma > 0$. According to the definition of $h_m$, we can write $h_m$ as $h_m = gc_m + f$, where $g= \frac{2\eta}{\beta} \frac{4\tau\rho^2(\rho^{\tau}-1)}{\rho-1} + \rho, f = \frac{\eta L^2}{2} \frac{4\tau\rho^2(\rho^{\tau}-1)}{\rho-1} + \frac{L\rho}{2}$ are constants.

First, we choose $\beta >\eta$, then both $g, f$ are bounded positive constants. We have
\begin{align}
	c_m = c_{m+1}(1+\beta\eta+2L^2\eta^2g) + 2L^2\eta^2f . \nonumber
\end{align}
Let $a =\beta\eta+2L^2\eta^2g$. Because $c_{\tilde{M}} = 0$, it is easy to get
\begin{align}
	c_0 = 2L^2\eta^2f\frac{(1+a)^{\tilde{M}}-1}{a}. \nonumber
\end{align}
We take $\tilde{M} = \lfloor\frac{1}{a}\rfloor \leq \frac{1}{a}$, then we have $c_0 \leq \frac{4L^2\eta^2f}{a}$ and
\begin{align}
	\gamma =    & \frac{\alpha}{2}(\eta - \frac{4c_0\eta}{\alpha\beta} -\frac{4gc_0}{\alpha}\eta^2 - \frac{4f}{\alpha}\eta^2) . \nonumber
\end{align}
As recommended in~\cite{sashanksvrg16}, we can take $\eta = \mu/n^{2/3}, \beta = v/n^{1/3}$ with $\eta < \beta$ (assuming $n$ is large). Then we can get $f = O(1), g = O(1), a = O(1/n)$. By choosing $\mu,v$ to satisfy $\frac{16L^2f\mu}{\alpha v^2}<1$ such that $\frac{4c_0}{\alpha\beta}<1$, it is easy to find that $\gamma = O(1/n^{2/3})>0$, $\tilde{M} = O(n)$. Hence, to get an $\epsilon$-local optimal solution, the computation complexity by all $p$ threads is $O(n^{\frac{2}{3}}/\epsilon)$, and the computation complexity of each thread is $O(\frac{n^{\frac{2}{3}}}{p\epsilon})$.

\section{Experiment}
To verify our theoretical results about Hogwild! and AsySVRG, we use a fully-connected neural network to construct a non-convex function. The neural network has one hidden layer with 100 nodes and the sigmoid function is used for the activation function. We use the soft-max output and a $L_2$ regularization for training. The loss function is:
\begin{align}
	f(\w,\b) = -\frac{1}{n}\sum_{i=1}^n \sum_{k=1}^K \1\{y_i=k\}\log  o_i^{(k)} + \frac{\lambda}{2}\|\w\|^2, \nonumber
\end{align}
where $\w$ is the weights of the neural network, $\b$ is the bias, $y_i$ is the label of instance $\x_i$, $o_i^{(k)}$ is the output corresponding to $\x_i$, $K$ is the total number of class labels.

We use two datasets: connect-4 and MNIST\footnote{https://www.csie.ntu.edu.tw/$\sim$cjlin/libsvmtools/datasets/} to do experiments and $\lambda = 10^{-3}$. We initialize $\w$ by randomly sampling from a Gaussian distribution with mean being 0 and variance being 0.01, and initialize $\b = 0$. During training, we use a fixed stepsize for both Hogwild! and AsySVRG. The stepsize is chosen from $\{0.1, 0.05, 0.01, 0.005, 0.001, 0.0005,0.0001\}$, and the best is reported. For the iteration number of the inner-loop of AsySVRG, we set $M=n/p$, where $p$ is the number of threads. The experiments are conducted on a server with 12 Intel cores and 64G memory.

Figure~\ref{fig:compare} illustrates the convergence property of both Hogwild! and AsySVRG.  The x-axis denotes the CPU time, where we set the CPU time that Hogwild! passes through the whole dataset once with one thread as 1 unit. The y-axis denotes the training loss. In this experiment, we run Hogwild! and AsySVRG with 10 threads. Hogwild!-10 and AsySVRG-10 denote the corresponding methods with 10 threads. It is easy to see that both Hogwild! and AsySVRG are convergent. Furthermore, AsySVRG is faster than Hogwild!. This is consistent with our theoretical results in this paper.

\begin{figure}[!thb]
\begin{center}
\subfigure[MNIST]{\includegraphics[width=1.625in]{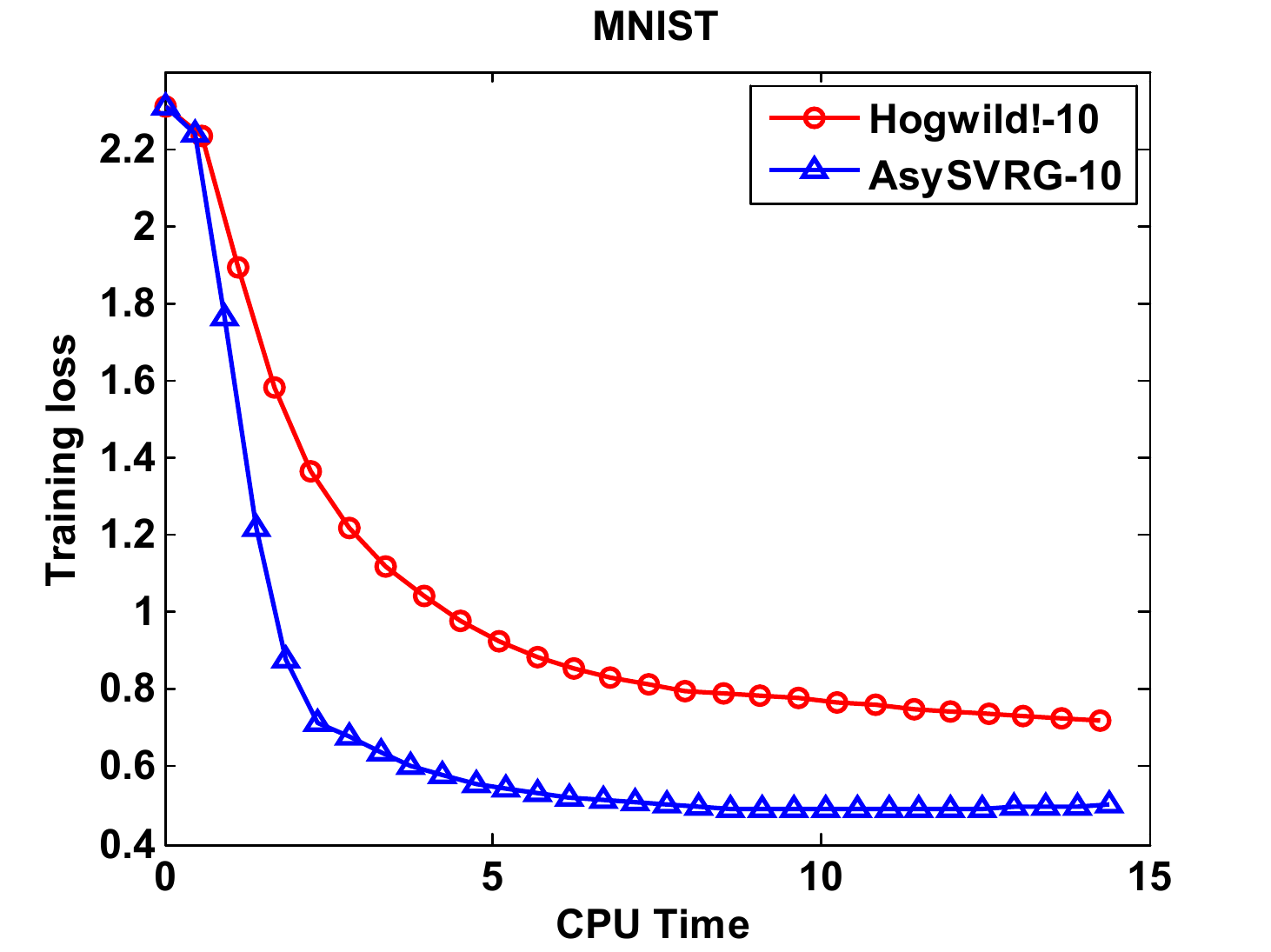}}
\subfigure[connect-4]{\includegraphics[width=1.625in]{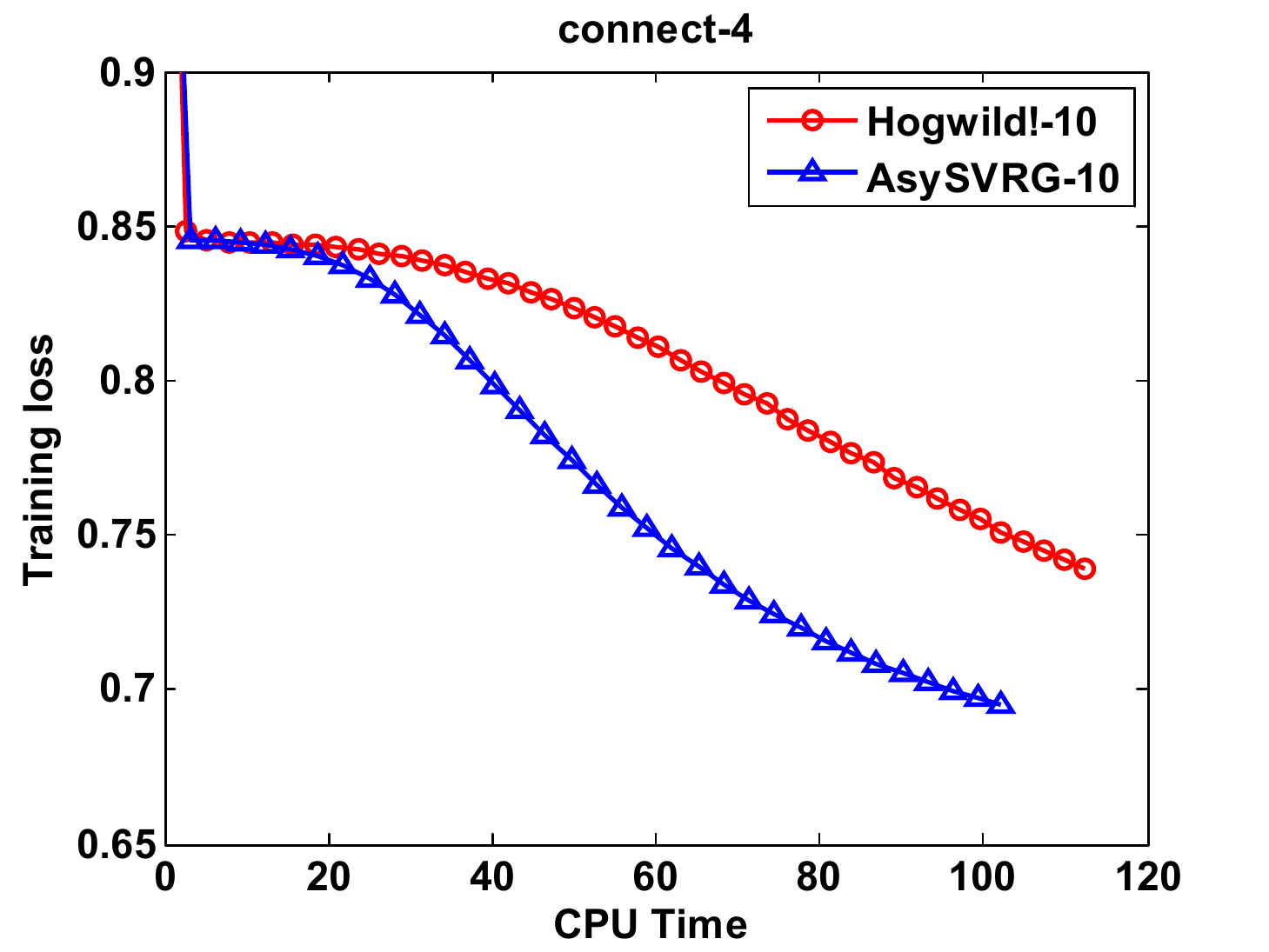}}
\end{center}
\caption{\small Hogwild! vs AsySVRG}
\label{fig:compare}
\end{figure}

Figure~\ref{fig:threads} reports the results of Hogwild! and AsySVRG with different numbers of threads, where the number of threads $p=1,4,10$. We can find that in most cases the two methods will become faster with the increase of threads. The only outlier is the case for  Hogwild! on dateset connect-4, Hogwild! using 4 threads is slower than using 1 thread. One possible reason is that we have two CPUs in our server, with 6 cores for each CPU. In the 4-thread case, different threads may be allocated on different CPUs, which will cause extra cost.

\begin{figure}[!hb]
\begin{center}
\subfigure[Hogwild! on MNIST]{\includegraphics[width=1.625in]{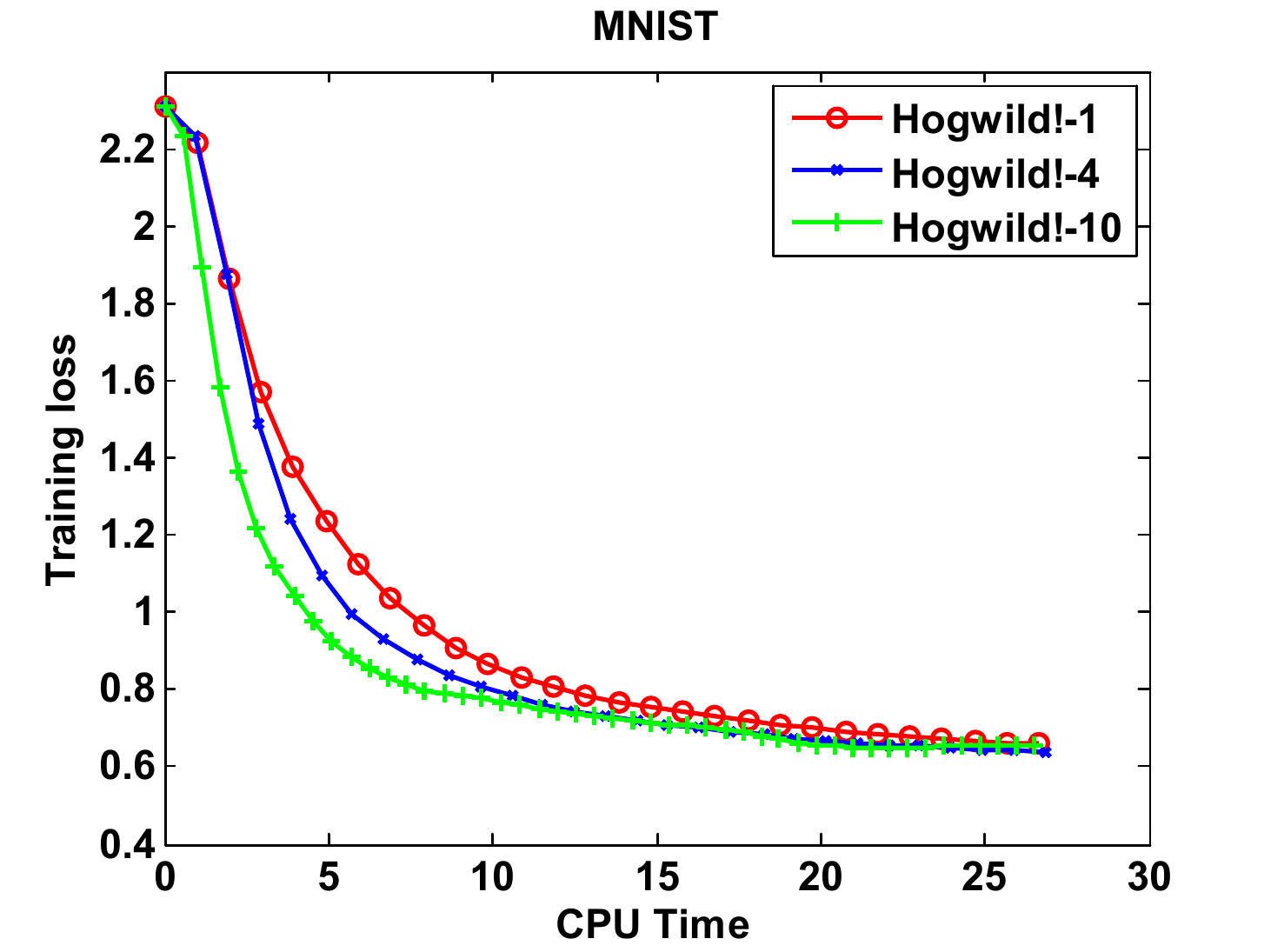}}
\subfigure[AsySVRG on MNIST]{\includegraphics[width=1.625in]{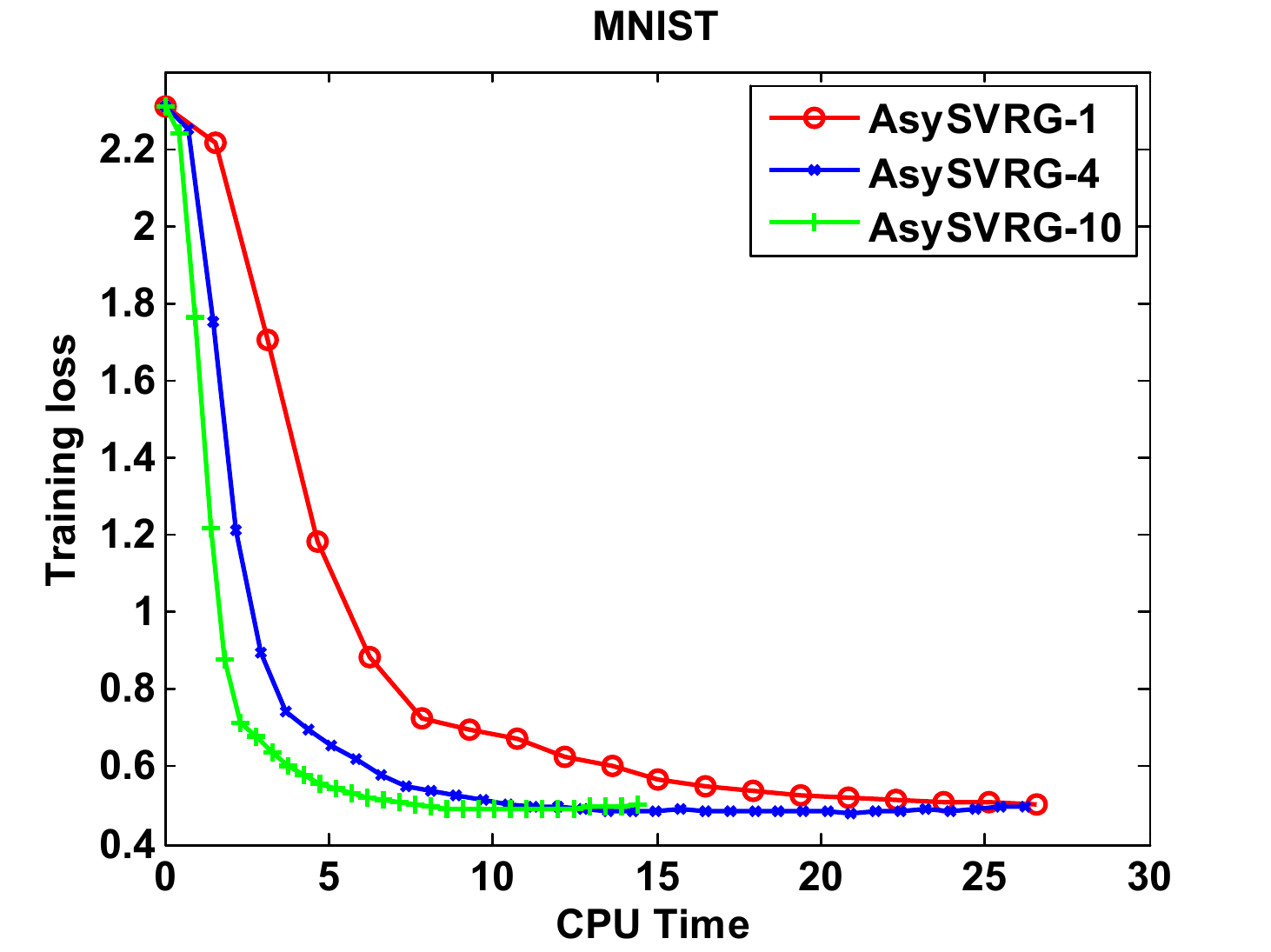}}
\subfigure[Hogwild! on connect-4]{\includegraphics[width=1.625in]{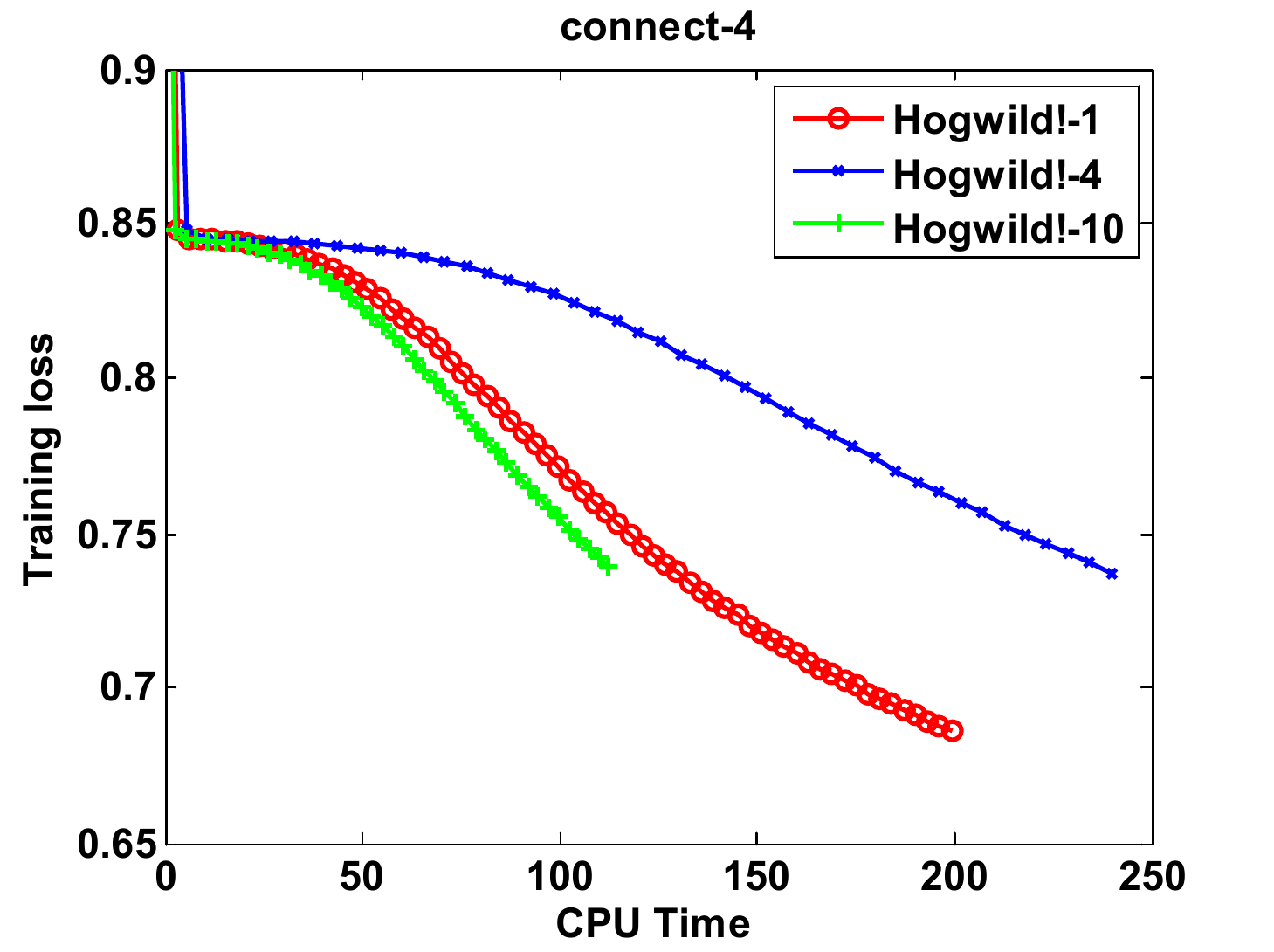}}
\subfigure[AsySVRG on connect-4]{\includegraphics[width=1.625in]{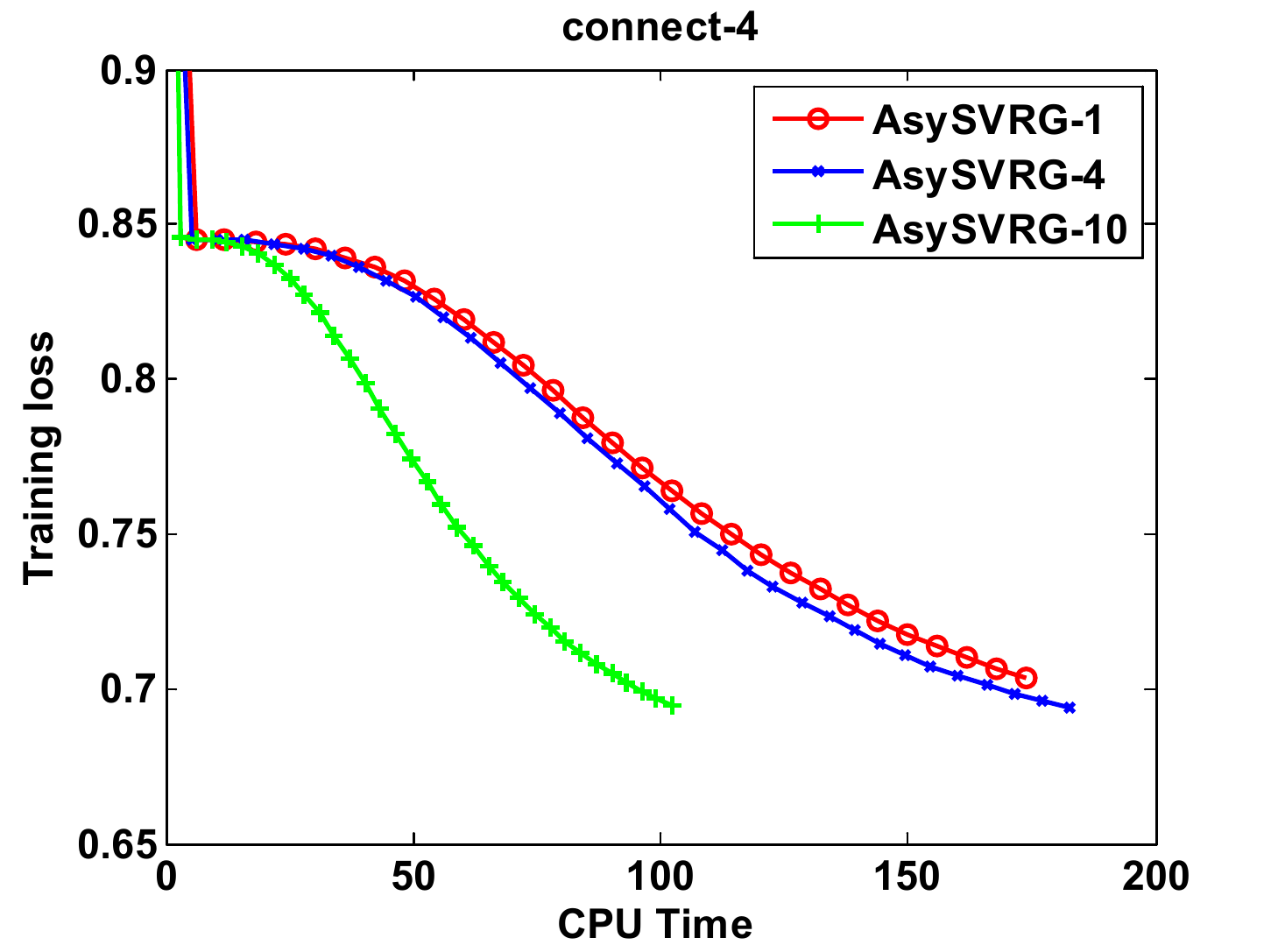}}
\end{center}
\caption{\small Comparison between different numbers of threads.}
\label{fig:threads}
\end{figure}

\section{Conclusion}
In this paper, we have provided theoretical proof about the convergence of two representative lock-free strategy based parallel SGD methods, Hogwild! and AsySVRG, for non-convex problems. Empirical results also show that both Hogwild! and AsySVRG are convergent on non-convex problems, which successfully verifies our theoretical results. To the best of our knowledge, this is the first work to prove the convergence of lock-free strategy based parallel SGD methods for non-convex problems.

\section{Acknowledgements}
{This work is partially supported by NSFC~(No.~61472182) and a fund from Tencent.}

\bibliographystyle{aaai}
\bibliography{ref}

\newpage
\appendix

\section{Appendix}

\subsection{Proof of Lemma \ref{techlemma:Hog}}

\begin{proof}
First, $\forall \x, \y$ and $r>0$, we have:
\begin{align}\label{basic eq}
     &\left\| \nabla f_{i}(\x) \right\|^2 - \left\| \nabla f_{i}(\y) \right\|^2 \nonumber \\
\leq &2\nabla f_{i}(\x)^T(\nabla f_{i}(\x) - \nabla f_{i}(\y)) \nonumber \\
\leq &\frac{1}{r}\left\| \nabla f_{i}(\x) \right\|^2 + r\left\| \nabla f_{i}(\x) - \nabla f_{i}(\y) \right\|^2 \nonumber \\
=    &\frac{1}{r}\left\| \nabla f_{i}(\x) \right\|^2 + r\left\| \nabla f_{i}(\x) - \nabla f_{i}(\y) \right\|^2 \nonumber \\
\leq &\frac{1}{r}\left\| \nabla f_{i}(\x) \right\|^2 + rL^2\left\| \x - \y \right\|^2
\end{align}
In the above equation, take $\x = \hat{\w}_t, \y=\hat{\w}_{t+1}$, we obtain:
\begin{align}
     &\left\| \nabla f_{i}(\hat{\w}_t) \right\|^2 - \left\| \nabla f_{i}(\hat{\w}_{t+1}) \right\|^2 \nonumber \\
\leq &\frac{1}{r} \left\| \nabla f_{i}(\hat{\w}_t) \right\|^2 + rL^2\left\| \hat{\w}_{t} - \hat{\w}_{t+1} \right\|^2 \nonumber
\end{align}
According to the definition of $\hat{\w}_t$, we have
\begin{align}
     &\left\| \hat{\w}_t - \hat{\w}_{t+1} \right\| \nonumber \\
=    &\| \w_{a(t)} - \eta\sum_{j=a(t)}^{t-1}\P_{t, j-a(t)}\nabla f_{i_j}(\hat{\w}_j) \nonumber \\
     &- (\w_{a(t+1)} - \eta\sum_{j=a(t+1)}^{t}\P_{t+1, j-a(t+1)}\nabla f_{i_j}(\hat{\w}_j)) \| \nonumber \\
\leq &\left\| \w_{a(t)} - \w_{a(t+1)} \right\| \nonumber \\
     &+ \eta \sum_{j=a(t)}^{t-1} \left\| \nabla f_{i_j}(\hat{\w}_j) \right\| + \eta \sum_{j=a(t+1)}^{t} \left\| \nabla f_{i_j}(\hat{\w}_j) \right\| \nonumber \\
\leq &\sum_{j=a(t)}^{a(t+1)-1} \left\| \w_j - \w_{j+1} \right\| \nonumber \\
     &+ \eta \sum_{j=a(t)}^{t-1} \left\| \nabla f_{i_j}(\hat{\w}_j) \right\| + \eta \sum_{j=a(t+1)}^{t} \left\| \nabla f_{i_j}(\hat{\w}_j) \right\| \nonumber \\
\leq &\eta\sum_{j=a(t)}^{a(t+1)-1} \left\| \nabla f_{i_j}(\hat{\w}_j) \right\| \nonumber \\
     &+ \eta \sum_{j=a(t)}^{t-1} \left\| \nabla f_{i_j}(\hat{\w}_j) \right\| + \eta \sum_{j=a(t+1)}^{t} \left\| \nabla f_{i_j}(\hat{\w}_j) \right\| \nonumber \\
\leq &3\eta \sum_{j=t-\tau}^{t} \left\| \nabla f_{i_j}(\hat{\w}_j) \right\| \nonumber.
\end{align}
Combining the two above equation, we obtain:
\begin{align}
     &\left\| \nabla f_{i}(\hat{\w}_t) \right\|^2 - \left\| \nabla f_{i}(\hat{\w}_{t+1}) \right\|^2 \nonumber \\
\leq &\frac{1}{r} \left\| \nabla f_{i}(\hat{\w}_t) \right\|^2 + 9r(\tau+1) L^2\eta^2 \sum_{j=t-\tau}^{t} \left\| \nabla f_{i_j}(\hat{\w}_j) \right\|^2 \nonumber
\end{align}
For any fixed $i$, we take expectation on the random index $i_j$, we obtain:
\begin{align}
     &\left\| \nabla f_{i}(\hat{\w}_t) \right\|^2 - \left\| \nabla f_{i}(\hat{\w}_{t+1}) \right\|^2 \nonumber \\
\leq &\frac{1}{r} \left\| \nabla f_{i}(\hat{\w}_t) \right\|^2 + 9r(\tau+1) L^2\eta^2 \sum_{j=t-\tau}^{t} q(\hat{\w}_j) \nonumber
\end{align}
Summing up $i$ from $1$ to $n$, we obtain:
\begin{align}
     &\EB q(\hat{\w}_t) - \EB q(\hat{\w}_{t+1}) \nonumber \\
\leq &\frac{1}{r}\EB q(\hat{\w}_t) + 9r(\tau+1) L^2\eta^2 \sum_{j=t-\tau}^{t} \EB q(\hat{\w}_j) \nonumber
\end{align}
Now, we prove the final result by induction and take $r = \frac{1}{\eta}$. When $t=0$, we have
\begin{align}
\EB q(\hat{\w}_0) \leq \frac{1}{1-\eta - 9\eta(\tau+1) L^2} \EB q(\hat{\w}_1) \leq \rho \EB q(\hat{\w}_1) \nonumber
\end{align}
Assuming the result is right for $t-1$, then we have
\begin{align}
     &\EB q(\hat{\w}_t) - \EB q(\hat{\w}_{t+1}) \nonumber \\
\leq &\frac{1}{r}\EB q(\hat{\w}_t) + 9\eta(\tau+1) L^2 \EB q(\hat{\w}_j)\sum_{j=t-\tau}^{t}\rho^{t-j} \nonumber
\end{align}
which means
\begin{align}
\EB q(\hat{\w}_t) \leq \frac{1}{1-\eta - \frac{9\eta(\tau+1) L^2(\rho^{\tau+1}-1)}{\rho-1}} \EB q(\hat{\w}_{t+1}) \leq \rho \EB q(\hat{\w}_{t+1}) \nonumber
\end{align}
\end{proof}

\subsection{Proof of Lemma \ref{Hoglem:gap of wr}}
\begin{proof}
\begin{align}
       & \|\w_t - \hat{\w}_t \| \nonumber \\
  =    & \|\w_t - \w_{a(t)} + \eta\sum_{j=a(t)}^{t-1}\P_{t, j-a(t)}\nabla f_{i_j}(\hat{\w}_j)\| \nonumber \\
  \leq & \|\w_t - \w_{a(t)}\| + \eta\sum_{j=a(t)}^{t-1}\| \nabla f_{i_j}(\hat{\w}_j)\| \nonumber \\
  \leq & \sum_{j=a(t)}^{t-1}\|\w_{j}-\w_{j+1}\| + \eta\sum_{j=a(t)}^{t-1}\| \nabla f_{i_j}(\hat{\w}_j)\| \nonumber \\
  \leq & 2\eta\sum_{j=a(t)}^{t-1}\|\nabla f_{i_j}(\hat{\w}_j)\| \nonumber
\end{align}
Then we get the result
\begin{align}
  \EB\|\w_t - \hat{\w}_t \|^2 \leq & 4\eta^2 \tau \sum_{j=a(t)}^{t-1}\EB\|\nabla f_{i_j}(\hat{\w}_j)\|^2 \nonumber \\
                              =    & 4\eta^2 \tau \sum_{j=a(t)}^{t-1}\EB q(\hat{\w}_j) \nonumber \\
                              \leq & \frac{4\eta^2 \tau\rho(\rho^{\tau}-1)}{\rho-1}\EB q(\hat{\w}_t)\nonumber
\end{align}
where the last inequality uses Lemma \ref{techlemma:Hog} in the appendix.
\end{proof}

\subsection{Proof of Lemma \ref{techlemma:svrg}}
\begin{proof}
First, $\forall \x, \y$ and $r>0$, we have:
\begin{align}\label{svrgbasic eq}
     &\left\| \p_{i}(\x) \right\|^2 - \left\| \p_{i}(\y) \right\|^2 \nonumber \\
\leq &2\p_{i}(\x)^T(\p_{i}(\x) - \p_{i}(\y)) \nonumber \\
\leq &\frac{1}{r}\left\| \p_{i}(\x) \right\|^2 + r\left\| \p_{i}(\x) - \p_{i}(\y) \right\|^2 \nonumber \\
=    &\frac{1}{r}\left\| \p_{i}(\x) \right\|^2 + r\left\| \p_{i}(\x) - \p_{i}(\y) \right\|^2 \nonumber \\
\leq &\frac{1}{r}\left\| \p_{i}(\x) \right\|^2 + rL^2\left\| \x - \y \right\|^2
\end{align}
In the above equation, take $\x = \hat{\u}_{t,m}, \y=\hat{\u}_{t,m+1}$, we obtain:
\begin{align}
     &\left\| \p_{i}(\hat{\u}_{t,m}) \right\|^2 - \left\| \p_{i}(\hat{\u}_{t,m+1}) \right\|^2 \nonumber \\
\leq &\frac{1}{r} \left\| \p_{i}(\hat{\u}_{t,m}) \right\|^2 + rL^2\left\| \hat{\u}_{t,m} - \hat{\u}_{t,m+1} \right\|^2 \nonumber
\end{align}
Furthermore, we can get
\begin{align}
     &\left\| \hat{\u}_{t,m} - \hat{\u}_{t,m+1} \right\| \nonumber \\
=    &\| \u_{t,a(m)} - \eta \sum_{i=a(m)}^{m-1} \P_{m,i-a(m)}^{(t)} \hat{\v}_{t,i} \nonumber \\
     &- (\u_{t,a(m+1)} - \eta \sum_{i=a(m+1)}^{m} \P_{m+1,i-a(m+1)}^{(t)} \hat{\v}_{t,i}) \| \nonumber \\
\leq &\left\| \u_{t,a(m)} - \u_{t,a(m+1)} \right\| + \eta \sum_{i=a(m)}^{m-1} \left\| \hat{\v}_{t,i} \right\| + \eta \sum_{i=a(m+1)}^{m} \left\| \hat{\v}_{t,i} \right\| \nonumber \\
\leq &\sum_{i=a(m)}^{a(m+1)-1} \left\| \u_{t,i} - \u_{t,i+1} \right\| + \eta \sum_{i=a(m)}^{m-1} \left\| \hat{\v}_{t,i} \right\| + \eta \sum_{i=a(m+1)}^{m} \left\| \hat{\v}_{t,i} \right\| \nonumber \\
\leq &\eta\sum_{i=a(m)}^{a(m+1)-1} \left\| \hat{\v}_{t,i} \right\| + \eta \sum_{i=a(m)}^{m-1} \left\| \hat{\v}_{t,i} \right\| + \eta \sum_{i=a(m+1)}^{m} \left\| \hat{\v}_{t,i} \right\| \nonumber \\
\leq &3\eta \sum_{i=m-\tau}^{m} \left\| \hat{\v}_{t,i} \right\| \nonumber.
\end{align}

Then, we take $r = \frac{1}{\eta}$ have
\begin{align}
     &\left\| \p_{i}(\hat{\u}_{t,m}) \right\|^2 - \left\| \p_{i}(\hat{\u}_{t,m+1}) \right\|^2 \nonumber \\
\leq &\eta \left\| \p_{i}(\hat{\u}_{t,m}) \right\|^2 + 9\eta (\tau+1) L^2\sum_{j=m-\tau}^{m} \left\| \hat{\v}_{t,j} \right\|^2 \nonumber
\end{align}
For any fixed $i$, we can take expectation for both sides of the above inequality and then sum $i$ from $1$ to $n$. Then we can get:
\begin{align}
     &\EB q(\hat{\u}_{t,m}) - \EB q(\hat{\u}_{t,m+1}) \nonumber \\
\leq &\eta \EB q(\hat{\u}_{t,m}) +  9\eta(\tau+1)L^2\sum_{j=m-\tau}^{m} \EB q(\hat{\u}_{t,j}) \nonumber
\end{align}
Here we use the fact that $\EB [\| \hat{\v}_{t,j} \|^2 | \hat{\u}_{t,j} ] = q(\hat{\u}_{t,j})$.

We prove our conclusion by induction. For convenience, we use $q_i$ to denote $\EB q(\hat{\u}_{t,i})$.

When $m=0$, we have
\begin{align}
q_0 \leq \frac{1}{1-\eta-9\eta(\tau+1)L^2} q_1 \leq \rho q_1 \nonumber
\end{align}
Assuming that $\forall m \leq M_0$, we have $q_{m-1} \leq \rho q_{m}$, then for $m = M_0$, we have
\begin{align}
q_m -& q_{m+1} \leq \eta q_m + 9\eta(\tau+1)L^2\sum_{j=m-\tau}^{m} q_j \nonumber \\
              \leq &\eta q_m + 9\eta(\tau+1)L^2 q_m \sum_{j=m-\tau}^{m} \rho^{m-j} \nonumber \\
              =    &\eta q_m + \frac{9\eta(\tau+1)L^2(\rho^{\tau+1}-1)}{\rho - 1}q_m \nonumber
\end{align}
which means that
\begin{align}
q_m \leq \frac{1}{1-\eta-\frac{9\eta(\tau+1)L^2(\rho^{\tau+1}-1)}{\rho-1}}q_{m+1} \leq \rho q_{m+1} \nonumber
\end{align}
\end{proof}

\subsection{Proof of Lemma \ref{svrglem:gap of wr}}
\begin{proof}
\begin{align}
     &\left\| \hat{\u}_{t,m} - \u_{t,m} \right\| \nonumber \\
   = &\left\| \u_{t,a(m)} - \eta \sum_{j=a(m)}^{m-1} \P_{m,j-a(m)}^{(t)} \hat{\v}_{t,j} - \u_{t,m} \right\| \nonumber \\
\leq &\left\| \u_{t,a(m)} - \u_{t,m} \right\| + \eta\sum_{j=a(m)}^{m-1} \left\| \hat{\v}_{t,j} \right\|  \nonumber \\
\leq &\sum_{j=a(m)}^{m-1} \left\| \u_{t,j} - \u_{t,j+1} \right\| + \eta \sum_{j=a(m)}^{m-1} \left\| \hat{\v}_{t,j} \right\|  \nonumber \\
\leq &\eta \sum_{j=a(m)}^{m-1} \left\| \hat{\v}_{t,j} \right\| + \eta \sum_{j=a(m)}^{m-1} \left\| \hat{\v}_{t,j} \right\|  \nonumber \\
\leq &2\eta\sum_{j=a(m)}^{m-1} \left\| \hat{\v}_{t,j} \right\| \label{hatu minus u}
\end{align}
Then we get the result
\begin{align}
  \EB\left\| \hat{\u}_{t,m} - \u_{t,m} \right\|^2 \leq &4\eta^2\tau \sum_{j=a(m)}^{m-1} \EB\left\| \hat{\v}_{t,j} \right\|^2 \nonumber \\
                                                     = &4\eta^2\tau \sum_{j=a(m)}^{m-1} \EB q(\hat{\u}_{t,j}) \nonumber \\
                                                  \leq &\frac{4\eta^2\tau \rho(\rho^\tau-1)}{\rho-1}\EB q(\hat{\u}_{t,m}) \nonumber
\end{align}
where the last inequality uses Lemma \ref{techlemma:svrg}.
\end{proof}

\subsection{Proof of Lemma \ref{svrglem:expqum}}
\begin{proof}
First, for any fixed $i$, we have
\begin{align}
     &\left\| \p_{i}(\hat{\u}_{t,m}) \right\|^2 - \left\| \p_{i}(\u_{t,m}) \right\|^2 \nonumber \\
\leq &2\p_{i}(\x)^T(\p_{i}(\hat{\u}_{t,m}) - \p_{i}(\u_{t,m})) \nonumber \\
\leq &\eta\left\| \p_{i}(\hat{\u}_{t,m}) \right\|^2 + \frac{1}{\eta}\left\| \p_{i}(\hat{\u}_{t,m}) - \p_{i}(\u_{t,m}) \right\|^2 \nonumber \\
=    &\eta\left\| \p_{i}(\hat{\u}_{t,m}) \right\|^2 + \frac{1}{\eta}\left\| \p_{i}(\hat{\u}_{t,m}) - \p_{i}(\u_{t,m}) \right\|^2 \nonumber \\
\leq &\eta\left\| \p_{i}(\hat{\u}_{t,m}) \right\|^2 + \frac{1}{\eta}L^2\left\| \hat{\u}_{t,m} - \u_{t,m} \right\|^2
\end{align}
Takeing expectation on both sides, using Lemma \ref{svrglem:gap of wr} and then summing up $i$ from $1$ to $n$ and we can get
\begin{align}
     &\EB q(\hat{\u}_{t,m}) - \EB q(\u_{t,m}) \nonumber \\
\leq &\eta\EB q(\hat{\u}_{t,m}) + \frac{4\eta\tau L^2\rho(\rho^\tau-1)}{\rho-1} \EB q(\hat{\u}_{t,m}) \nonumber
\end{align}
which means that
\begin{align}
\EB q(\hat{\u}_{t,m}) \leq \frac{1}{1-\eta-\frac{4\eta\tau L^2\rho(\rho^\tau-1)}{\rho-1}}\EB q(\u_{t,m}) < \rho \EB q(\u_{t,m}) \nonumber
\end{align}
\end{proof}

\end{document}